\documentclass[letterpaper]{article} % DO NOT CHANGE THIS
\usepackage{aaai2027}  % DO NOT CHANGE THIS
\nocopyright
\usepackage{xcolor}
\usepackage[hyphens]{url}  % DO NOT CHANGE THIS
\newcommand{\KL}{\mathrm{KL}}
\newcommand{\one}{\mathbf{1}}
\usepackage{graphicx} % DO NOT CHANGE THIS
\usepackage{natbib}  % DO NOT CHANGE THIS AND DO NOT ADD ANY OPTIONS TO IT
\usepackage{caption} % DO NOT CHANGE THIS AND DO NOT ADD ANY OPTIONS TO IT
\definecolor{addblue}{RGB}{0,60,180} 
\definecolor{addamber}{RGB}{190,130,0} 

\usepackage{algorithm}
\usepackage{algorithmic}

\usepackage{newfloat}
\usepackage{listings}
\DeclareCaptionStyle{ruled}{labelfont=normalfont,labelsep=colon,strut=off} % DO NOT CHANGE THIS
\floatstyle{ruled}
\newfloat{listing}{tb}{lst}{}
\floatname{listing}{Listing}

\usepackage{booktabs}

\usepackage{enumitem}

\usepackage{graphicx,amsfonts,amsmath,amssymb,amsthm,natbib,blkarray,bm}

\title{\otap{}: Structure-Aware Optimal Transport for \\
Evaluating Planning and Execution in Agent Trajectories}

\author{
    Babak Barazandeh\textsuperscript{\rm 1}\thanks{Corresponding author: bbarazandeh@cribl.io},
    Subhabrata Majumdar\textsuperscript{\rm 2},
    George Michailidis\textsuperscript{\rm 3}
}
\affiliations{
    \textsuperscript{\rm 1}Cribl AI Research Lab,
    \textsuperscript{\rm 2}IIMB, 
    \textsuperscript{\rm 3}UCLA\\
    
}
\newtheorem{assumption}{Assumption}%[section]
\newtheorem{definition}{Definition}%[section]
\newtheorem{proposition}{Proposition}

\newtheorem{remark}{Remark}

\def\R{\mathbb{R}}

\def\L{\mathcal{L}}

\newcommand{\norm}[1]{\left\| #1 \right\|}

\newcommand{\T}{\mathbf{T}}
\newcommand{\Cost}{\mathbf{C}}
\newcommand{\otap}{\textsc{Otap}}

\usepackage{soul}
\definecolor{addgreen}{RGB}{0,120,0}

\begin{document}

\maketitle

\begin{abstract}
Large language model agents solve tasks by generating trajectories that
interleave planning, tool calls, and intermediate results. Current
evaluation metrics reduce such a trajectory to a binary success flag,
compare it against a reference by exact matching, or delegate judgment to
another language model. A success flag cannot distinguish a sound solution
from one that succeeds by luck, and says nothing about why a failed run
went wrong. Exact matching penalizes plans that are valid but reordered or
decomposed differently from the reference. We reframe trajectory evaluation
as a distance between the agent's execution graph and a set of valid
solution graphs, and instantiate it via an unbalanced fused
Gromov-Wasserstein transport problem over attributed dependency graphs.
The resulting score, termed \otap{} (Optimal Transport for Agentic
Planning), is a normalized discrepancy
score. Its idealized, globally minimized entropy-free objective is
invariant to dependency-preserving relabelings of a fixed graph and has
a redundant-insertion sensitivity analysis under stated assumptions. Its unbalanced
marginals handle missing or hallucinated steps without forcing a
match; granularity invariance remains an open problem that
soft coupling mitigates only in part. On controlled perturbations and
three public benchmarks, \otap{} separates valid from invalid trajectories
in a regime where semantics-only metrics score below chance.
Its performance is encouraging across these settings, while controlled
edge-noise experiments reveal sensitivity to dependency extraction. Where a formal verifier exists, strict surface
metrics predict validity better than \otap{} does, which places \otap{}
in open-ended domains where no verifier is available. A blinded human study of 236 trajectories finds the strongest agreement
with \otap{} among the six metrics in that comparison. Structure-aware
baselines achieve higher clean-graph AUROC, while \otap{} has the
smallest absolute AUROC decline under the tested edge noise.
\end{abstract} 

\section{Introduction}

Agentic systems built on large language models (LLMs) decompose open-ended natural language goals into distinct execution trajectories. These trajectories interleave high-level planning, explicit tool calls, and dynamic environment interactions to manage intermediate execution states. Contemporary frameworks span explicit plan-and-execute pipelines, ReAct-style interleaved reasoning and acting pipelines~\cite{yao2023react}, and tree- or graph-structured deliberation graphs~\cite{yao2023tot}. As agent architectures increasingly transition into production environments, establishing rigorous, process-level evaluation methodologies has become a primary bottleneck for both academic research and systematic deployment.

Agent evaluation uses three overlapping approaches:
\begin{itemize}[nolistsep,leftmargin=*]
\item \textbf{Outcome-based evaluation} checks task completion, as in
WebArena~\cite{zhou2024webarena}, OSWorld~\cite{otapXie2024}, and
$\tau^2$-Bench~\cite{otapBarres2026}. Final success alone does not describe
which intermediate actions or dependencies were satisfied; these
benchmarks may also provide other diagnostics.
\item \textbf{Reference-based evaluation} compares outputs, tool calls,
or task graphs with references. TaskBench~\cite{otapShen2024} evaluates
node and edge predictions, BFCL~\cite{otapPatil2025} includes AST-based
function-call evaluation, and TRAJECT-Bench~\cite{otapHe2025} checks tool
selection, arguments, and dependency/order satisfaction. These methods
already go beyond literal string matching; comparison across alternative
reference plans and step decompositions remains a distinct design issue.
\item \textbf{LLM- and agent-based evaluation} uses model judgments or
active evidence gathering, as in AgentRewardBench~\cite{otapLu2025Reward},
Agent-as-a-Judge~\cite{otapZhuge2025}, and AJ-Bench~\cite{otapShi2026}.
Such evaluators can assess execution processes, but introduce model
inference costs and dependence on the judge and evaluation protocol.
\end{itemize}

For realistic agentic tasks, there are multiple valid solutions, for example those differing by dependency-preserving reorderings, step granularity, or interchangeable tools. Hence agent evaluation metrics should measure closeness to this solution \emph{space} rather than identity with one reference sequence. Because of this reason, we argue that the central evaluation paradigm should shift from checking whether an agent reproduced an exact sequence of reference steps to a holistic graph alignment question:

\begin{quote}
\emph{How close is the agent's execution graph to the set of valid solution graphs for the task?}
\end{quote}

Answering this question requires trajectory comparison metrics that satisfy a few criteria that non-unique but valid solutions to the same task satisfy: (i) \emph{semantic awareness} to accommodate paraphrased steps, (ii) \emph{permutation tolerance paired with causal strictness} to permit independent step reorderings but penalize true causal inversions, (iii) \emph{granularity tolerance} to support one-to-many step mappings, (iv) \emph{cardinality robustness} to handle missing or hallucinated steps, (v) \emph{execution-grounding} to incorporate tool environments and state changes, and (vi) \emph{multi-reference support} to reflect diverse valid planning strategies. Appendix~\ref{app:criteria} grounds each
criterion in a documented failure of an existing metric family and names
the experiment that tests it; each criterion addresses a failure mode of at least one comparison
method. These are task-motivated desiderata; we do not claim they are
orthogonal or universally necessary for every evaluation setting.

A standard baseline aligns agent and reference steps via the Hungarian
algorithm~\cite{wright1990speeding} or greedy embedding matching, as in
text-similarity metrics~\cite{zhang2020bertscore}. A node-cost-only
assignment does not optimize dependency agreement; unequal graph sizes
require rectangular matching or dummy padding. We combine unbalanced
marginals~\cite{chizat2018scaling}, soft correspondences, and a fused
Gromov--Wasserstein structural term~\cite{vayer2019fgw}. Soft mass
splitting permits one-to-many correspondences but does not establish
granularity invariance. More specifically,

\begin{enumerate}[nolistsep,leftmargin=*]
\item We formalize agent trajectory evaluation as a divergence computation between \emph{attributed dependency graphs} and define criteria for valid trajectory metrics (Sec.~\ref{sec:problem}).
\item We introduce \otap{}, instantiating this divergence via an unbalanced fused Gromov-Wasserstein (UFGW) formulation that couples action, argument, effect, tool, and state-transition attributes with signed dependency structures (Sec.~\ref{sec:method}).
\item We enable \otap{} to reference ensembles via a temperature-controlled soft-minimum aggregation, using causal mass weighting to emphasize important path dependencies (Sec.~\ref{sec:multiref}).
\item We establish theoretical guarantees for \otap{}, including non-negativity and symmetry of an idealized symmetrized discrepancy, fixed-graph permutation invariance, a redundant-insertion sensitivity analysis, and per-iteration computational costs (Sec.~\ref{sec:theory}).
\item We empirically validate \otap{} on a set of hand-curated tasks and public benchmarks (Sec.~\ref{sec:experiments}).% (\textbf{STILL PENDING.
\end{enumerate}

\section{Related Work}

\textbf{Agent benchmarks.} Existing evaluation suites like AgentBench~\cite{liu2023agentbench}, GAIA~\cite{mialon2023gaia}, WebArena~\cite{zhou2024webarena}, ToolBench~\cite{qin2023toolllm}, and $\tau$-bench~\cite{yao2024taubench} primarily emphasize environmental task execution. PlanBench~\cite{valmeekam2023planbench} utilizes formal verification to assess planning accuracy, which provides exact ground truth but remains restricted to highly structured deterministic domains. These benchmarks do not offer a graded, open-ended distance metric between structural execution graphs. Some recent papers propose metrics beyond accuracy to measure agent performance, such as confidence calibration \cite{zhang2026agenticconfidencecalibration}, uncertainty quantification \cite{tayebati2026tracertrajectoryriskaggregation}, and consistency measurement \cite{raj2026consistencytestablepropertystatistical}.

\textbf{Process-level evaluation.} Process reward models (PRMs) evaluate intermediate reasoning steps to optimize mathematical and logical execution~\cite{lightman2024let}. However, PRMs require extensive step-level manual annotations and act as black-box learned models. Similarly, LLM-as-a-judge methods~\cite{zheng2023judge} inherit high inference costs and non-deterministic variations. In contrast, \otap{} is deterministic given text embeddings, is computationally lightweight, and provides explicit sub-loss decompositions.

\textbf{Graph-based trajectory evaluation.}
TaskBench~\cite{otapShen2024} scores tool nodes and dependency edges.
ToolSandbox~\cite{otapLu2025Sandbox} matches trajectory snapshots to a
milestone DAG subject to dependency-compatible order, already tolerating
valid reorderings and providing partial-progress diagnostics. \otap{}
compares attributed trajectory graphs through a joint semantic and
structural transport objective and aggregates over multiple references.
Thus dependency evaluation and reordering tolerance alone are not novel;
Sec.~\ref{sec:structbase} compares the resulting formulation with
TaskBench-style F1 and approximate graph edit distance.

\textbf{Text and graph similarity.}
Lexical and embedding comparisons~\cite{zhang2020bertscore} do not by
themselves encode causal dependencies. Graph edit distance
(GED)~\cite{sanfeliu1983ged,otapGao2010} handles structural differences
through node/edge edits, including insertions and deletions, and can use
semantic substitution costs. It differs from \otap{} in using discrete
edit correspondences rather than soft transport with marginal penalties;
missing-step handling is not unique to \otap{}.

\textbf{Optimal transport.}
Entropic OT and unbalanced OT provide computational and mass-relaxation
tools~\cite{cuturi2013sinkhorn,chizat2018unbalanced}. Fused GW jointly
compares graph features and structure~\cite{vayer2019fgw}, while
unbalanced GW relaxes mass conservation~\cite{otapSejourne2021}.
\otap{} adapts this machinery to agent dependency graphs and reference
sets. Metric properties depend on the precise formulation, as emphasized
by work on partial GW~\cite{otapBai2025}; they do not automatically
transfer to our operational discrepancy.

\section{Problem Formulation}\label{sec:problem}

\begin{definition}[Trajectory graph]
An agent execution trajectory is represented as an attributed directed acyclic graph $G=(V,E,\phi)$, where each vertex $s_i \in V$ defines an executed step:
\[
s_i = (a_i,\ \tau_i,\ c_i,\ e_i),
\]
comprising an action description $a_i$, an invoked tool identifier $\tau_i \in \mathcal{T}$ (where $\tau_i = \varnothing$ denotes an internal reasoning step), specific contextual arguments $c_i$, and an execution effect or post-state artifact $e_i$. A directed edge $(s_i, s_k) \in E$ indicates a strict data or causal dependency where step $s_k$ consumes an artifact or requires the prior execution of step $s_i$. The mapping function $\phi$ projects step fields into a continuous vector space.
\end{definition}

Edges are systematically constructed via three mechanisms: (a) extracting explicit dependency graphs from plan-and-execute planners; (b) tracing artifact dependencies in ReAct sequences by linking actions to recent observations; or (c) reverting conservatively to a total execution order when structural dependencies are unrecoverable.

Let $P = (V^P, E^P)$ denote the agent trajectory with $|V^P| = n$, and let $\mathcal{R} = \{R_1,\dots,R_K\}$ represent a set of diverse reference trajectories where $|V^{R_k}| = m_k$. The evaluation problem requires computing a normalized distance score $\mathrm{score}(P, \mathcal{R}) \in [0,1]$ that faithfully reflects structural and semantic deviations from the valid solution space.

\section{The \otap{} Metric}\label{sec:method}
We now formalize Optimal Transport for Agentic Planning (\otap{}) to quantify the similarity between an agent trajectory $P$ and a reference set $\mathcal R$. Instead of treating agent trajectories as flat text strings or relying on binary success criteria, \otap{} models execution histories as directed acyclic graphs (DAGs) capturing both semantic content and execution dependencies.

\subsection{The optimization problem}
We begin by calculating \otap{} between two \textit{individual} trajectories (say $P,R$) in a three-stage process: (i) computing node-level semantic distances between their individual execution steps, (ii) constructing dependency matrices to capture the topology of the plans, and (iii) solving an OT problem that globally aligns the trajectories.

\paragraph{Step representation and node cost}
To prevent argument strings from dominating semantic representations, step components are embedded independently using a text encoder $\psi(\cdot)$: $\mathbf{a}_i = \psi(a_i)$, $\mathbf{c}_i = \psi(c_i)$, and $\mathbf{e}_i = \psi(e_i)$. The local node-level cost matrix between agent step $i$ and reference step $j$ is defined as:
\begin{align}
C_{ij} \;=\;& \alpha\, d\!\left(\mathbf{a}_i, \mathbf{a}_j\right)
        + \beta\, d\!\left(\mathbf{c}_i, \mathbf{c}_j\right)
        + \gamma\, d\!\left(\mathbf{e}_i, \mathbf{e}_j\right) \nonumber\\
       &+ \delta\, d_{\mathcal{T}}(\tau_i, \tau_j)
        + \eta\, d_{\mathrm{state}}(i,j),
\label{eq:nodecost}
\end{align}
where $d$ denotes cosine distance and the hyper-parameters satisfy $\alpha+\beta+\gamma+\delta+\eta = 1$. The tool dissimilarity term $d_{\mathcal{T}}$ evaluates functional equivalence (e.g., scoring $0$ for identical tools, a fractional value (default $0.5$) for substitutable tools like a Python interpreter vs.\ a calculator, and $1$ for completely incompatible domains). When environment state logs $x_t$ are available, $d_{\mathrm{state}}(i,j)$ evaluates world state congruence via direct property comparison or DOM/file serialization embeddings, providing an execution-grounded proxy. When states are omitted, we set $\eta = 0$.

The tool costs form a symmetric lookup declared once per corpus and held
fixed across evaluations. Table~\ref{tab:toolcost} specifies its tiers.
\begin{table}[t]
\centering
\caption{Tool-dissimilarity lookup rule.}
\label{tab:toolcost}
\small
\begin{tabular}{lc}
\toprule
Tool relationship & Cost \\
\midrule
Identical tool & $0$ \\
Functionally substitutable (e.g., Python/calculator) & $0.5$ \\
Incompatible categories & $1$ \\
\bottomrule
\end{tabular}
\end{table}

\paragraph{Structure representation}
To model plan topologies, we construct a normalized, signed pairwise dependency distance matrix for each graph. For the agent graph $P$, let $D^P_{ik}$ evaluate the structural distance from step $s_i$ to $s_k$ relative to the maximum DAG height $H$:
\begin{equation}
D_{ik} =
\begin{cases}
+\,\rho(i,k)/H & i \prec k \ \text{(ancestor)},\\
0 & i \parallel k \ \text{(incomparable)},\\
-\,\rho(k,i)/H & k \prec i \ \text{(descendant)},
\end{cases}
\label{eq:depdist}
\end{equation}
where $\rho$ defines the directed shortest path hop distance. The sign preserves causal direction. Consequently, mapping a true causal sequence onto an inverted sequence yields a substantial quadratic penalty within the structural optimization term, whereas permuting order-independent steps (where $D=0$) incurs no penalty.

\paragraph{Unbalanced Fused Gromov-Wasserstein objective}
Let $\mu_i = 1/n$ and $\nu_j = 1/m$ define the initial marginal mass distributions over the agent and reference nodes, respectively. \otap{} minimizes an unbalanced fused Gromov-Wasserstein objective to compute the optimal coupling matrix $\T \in \R_{\ge 0}^{n \times m}$:
\begin{align}
F_0(\T;P,R) &=
(1-\theta) \sum_{i,j} C_{ij} T_{ij} \nonumber\\
&+ \theta \sum_{i,k,j,l} \big| D^P_{ik} - D^R_{jl} \big|^2\, T_{ij}\, T_{kl}
\nonumber\\
&+ \lambda_1\, \mathrm{KL}\!\left(\T \mathbf{1} \,\|\, \mu\right)
 + \lambda_2\, \mathrm{KL}\!\left(\T^{\!\top} \mathbf{1} \,\|\, \nu\right)
,
\label{eq:ufgw}
\end{align}
where $\theta\in[0,1]$ trades local attribute alignment against structural
matching. Entropy $H(\T)=-\sum_{ij}T_{ij}(\log T_{ij}-1)$ smooths the
inner unbalanced-OT solves. Let $\widehat\T_\varepsilon(P,R)$ denote the
coupling returned by the fixed solver, initialization, stopping rule,
and seed. The operational loss is
\[
\mathcal L(P,R)=F_0(\widehat\T_\varepsilon(P,R);P,R)\ge0.
\]
Every reported loss excludes $-\varepsilon H(\T)$; entropy is used to
obtain the coupling, not to score it. The nonconvex solve is not claimed
to find a global optimum.

This formulation overcomes the limitations of one-to-one matching baselines via three behaviors:
\begin{itemize}
    \item \emph{Flexible step granularity:} The continuous coupling matrix $\T$ permits soft mass splitting, allowing a singular high-level reference step to align across multiple decomposed agent sub-steps.
    \item \emph{Robustness to hallucinations and omissions:} The Kullback-Leibler (KL) marginal penalties relax strict mass conservation constraints. Unmatched reference mass represents missing actions (recall failure), while unaligned agent mass captures hallucinated steps (precision failure). Trajectory-level metrics are extracted via:

    \begin{equation}
    \pi = \frac{\sum_i \min\{(\T\mathbf{1})_i,\, \mu_i\}}{\|\mu\|_1},
    \qquad
    \rho = \frac{\sum_j \min\{(\T^{\!\top}\mathbf{1})_j,\, \nu_j\}}{\|\nu\|_1}.
    \end{equation}
    Clipping prevents surplus mass at one step from compensating for a
    deficit at another; the two coverage summaries need not coincide.
    Agent-side and reference-side deficits identify candidate extraneous
    and missing steps, respectively. A deficit is a diagnostic signal,
    not by itself proof of a semantic hallucination or omission.

    \item \emph{Causal topology alignment:} By measuring structural variance over signed distance spaces, structural orientation is optimized directly inside the transport alignment.
\end{itemize}
We solve \eqref{eq:ufgw} using a conditional gradient scheme. At each outer iteration, the quadratic GW term is linearized, reducing the inner loop to an unbalanced entropic optimal transport problem resolved via generalized Sinkhorn iterations.

\begin{remark}[Hierarchical Execution]
When hierarchical execution data is provided by the planner, \otap{} applies Eq.~\eqref{eq:ufgw} recursively. Sub-step alignment costs are aggregated to define parent-node distances, ensuring top-level tasks match effectively if their underlying child DAGs exhibit low structural transport costs.    
\end{remark}

\subsection{Multiple references and causal mass}\label{sec:multiref}

To handle diverse, equally valid plan structures, \otap{} aggregates scores across a reference set $\mathcal{R}$ using a temperature-controlled softmin function:
\begin{equation}
\mathcal{L}(P,\mathcal{R}) = -\,T \log \frac{1}{K} \sum_{k=1}^{K}
\exp \left[ -\frac{\mathcal{L}(P, R_k)}{T} \right],
\label{eq:softmin}
\end{equation}
which converges to the minimum as $T\to0$ and the arithmetic mean as
$T\to\infty$. Soft-min is smooth in its input losses, which is useful
when considering a process-reward surrogate. A hard minimum's value is
continuous, but its selected reference and gradient can change abruptly
at ties. For $L_{\min}=\min_k\mathcal L(P,R_k)$,
\[
L_{\min}\le\mathcal L(P,\mathcal R)\le L_{\min}+T\log K.
\]
The default $T=0.05$ therefore gives a near-min aggregation (an upper
bound of approximately $0.055$ for $K=3$). Raising $T$ to $1.0$ lowers
AUROC from $82.0$ to $79.5$, a $2.5$-point decline
(Table~\ref{tab:ablations}). Smooth aggregation alone does not establish
differentiability through every discrete preprocessing or solver step.

To prevent trivial utility steps from skewing evaluation, we introduce non-uniform target marginals based on a step's structural criticality:
\begin{equation}
\nu_j \propto 1 + \kappa \cdot \mathrm{crit}(j),
\end{equation}
where $\mathrm{crit}(j) \in \{0,1\}$ indicates whether removing vertex $j$ breaks path reachability between sources and goals within the reference DAG. This ensures that missing a critical step incurs a higher penalty.

\subsection{Final score}

Finally, we report the aggregate similarity score as
\begin{equation}
\mathrm{\otap{}}(P,\mathcal{R}) \;=\; \exp\!\big(-\mathcal{L}(P, \mathcal{R})\big) \in (0,1],
\label{eq:score}
\end{equation}
where each $\mathcal{L}(P,R_k)$ entering the soft-min of
Eq.~\eqref{eq:softmin} is the entropy-free value defined above; the range
$(0,1]$ therefore holds by construction. The score is paired with sub-loss diagnostics, precision/recall metrics $(\pi, \rho)$, and the optimal coupling matrix $\T^\star$ to provide an auditable execution alignment map. A complete worked example --- a valid reordering and a causal inversion of the same task, with their dependency matrices, optimal couplings, and sub-loss decompositions --- is provided in Appendix~\ref{app:worked}. Algorithm~\ref{alg:otap} summarizes the complete procedure.

\begin{algorithm}[t]
\caption{\otap{} evaluation }
\label{alg:otap}
\begin{algorithmic}[1]
\REQUIRE agent trace $P$, references $\mathcal{R}$, encoder $\psi$,
weights $(\alpha,\beta,\gamma,\delta,\eta,\theta,\lambda_{1,2},\varepsilon,T)$
\STATE Parse $P$ into steps $(a_i,\tau_i,c_i,e_i)$; build DAG $E^P$ by
artifact tracing; compute $D^P$ via \eqref{eq:depdist}
\FOR{$k = 1$ \TO $K$}
  \STATE Build $D^{R_k}$; embed fields with $\psi$; form $\Cost^{(k)}$
  via \eqref{eq:nodecost}
  \STATE $\widehat\T_k \leftarrow$ conditional-gradient solve with
  unbalanced entropic Sinkhorn inner loops
  \STATE $\mathcal L_k\leftarrow F_0(\widehat\T_k;P,R_k)$ via
  \eqref{eq:ufgw}, excluding entropy
\ENDFOR
\STATE $\mathcal{L} \leftarrow$ soft-min over $\{\mathcal{L}_k\}$ via
\eqref{eq:softmin}
\RETURN score \eqref{eq:score}, precision/recall $(\pi,\rho)$,
component losses, best transport plan $\T^\star$
\end{algorithmic}
\end{algorithm}

\section{Theoretical Properties}\label{sec:theory}

Only in this section and Appendix~\ref{app:proofs}, write
$\mathcal L(P,R)=\min_{\T\ge0}F_0(\T;P,R)$ for the idealized global
minimum of the entropy-free objective. This differs from the operational
loss evaluated at the computed coupling. The following claims concern
the idealized quantity; they are not global-optimality guarantees for
the numerical solver.

\begin{proposition}[Properties of the symmetrized discrepancy]
\label{prop:pseudo}
Let $d(P,R) = \tfrac12\big(\mathcal{L}(P,R) + \mathcal{L}(R,P)\big)$. Then $d \ge 0$, $d(P,P) = 0$, $d$ is symmetric, and $d(P,R) = 0$ if there exists an attribute-preserving and dependency-isomorphic mapping between $P$ and $R$.
\end{proposition}
\begin{proof}[Proof sketch]
Non-negativity holds directly from the component bounds. For identical graphs $P = R$, setting $\T = \mathrm{diag}(\mu)$ zeroes out the attribute cost since $C_{ii} = 0$, zeroes out the structural variance since $D^P = D^R$, and minimizes the KL divergence to zero since the marginals match perfectly. Symmetry and graph-isomorphism conditions follow directly from the structural invariance of the matrix operations under isomorphic vertex permutations.
\end{proof}

\begin{proposition}[Dependency-respecting permutation invariance]
\label{prop:perm}
Let $P'$ be an execution sequence generated by a valid topological sort of the DAG $E^P$. Then $\mathcal{L}(P',R) = \mathcal{L}(P,R)$ for any reference graph $R$.
\end{proposition}
\begin{proof}[Proof sketch]
The construction of the structural matrix $D^P$ in Eq.~\eqref{eq:depdist} and the local cost computation in Eq.~\eqref{eq:nodecost} depend strictly on the invariant topological reachability relations of the DAG, remaining completely independent of the specific serialization order chosen for execution.
\end{proof}

\begin{proposition}[Bounded sensitivity to redundant insertion]
\label{prop:robust}
Let $P^{+r}$ represent an agent trajectory expanded by inserting $r$ redundant leaf nodes that carry no downstream dependencies, under renormalized uniform mass allocations. Then:
\[
\big|\mathcal{L}(P^{+r},R) - \mathcal{L}(P,R)\big|
\;\le\; \frac{r}{n+r}\,\Big( \lambda_1\, \omega + (1-\theta)\, \bar{C}
+ 4\,\theta \Big),
\]
where $\bar{C} \le 1$ bounds node attribute costs and $\omega$ bounds the maximum per-unit-mass KL destruction cost. Score variation scales strictly as $O(r/n)$.
\end{proposition}
\begin{proof}[Proof sketch]
The proof proceeds by extending the OT coupling matrix computed for $(P,R)$ and leaving the $r$ newly introduced leaf nodes completely unassigned. The resulting objective inflation is strictly bounded by the explicit KL cost of leaving $r/(n+r)$ mass unaligned plus minor renormalization perturbations over the remaining structural links.
\end{proof}

\paragraph{Computational Complexity}
Mapping attributes requires $O(n+m)$ embedding generation passes. Structural distance calculation via all-pairs reachability over the DAG scales as $O(n^3)$ in the worst case, or $O(n \cdot |E|)$ under sparse conditions. For the UFGW optimization, each conditional-gradient pass linearizes the relational term in $O(n^2 m + n m^2)$ operations, while each inner Sinkhorn iteration requires $O(nm)$ steps. These are per-iteration costs, not a polynomial-time guarantee for the
global optimum. Conditional-gradient methods target stationary solutions
under appropriate convergence conditions; approximate inner solves and
finite stopping do not certify a global minimum. The reported computation
is deterministic given the encoder, solver settings, initialization, and
seed. Measured runtimes are reported in Sec.~\ref{sec:runtime}.

\section{Experiments}\label{sec:experiments}

This section reports seven empirical analyses, followed by runtime measurements. The first isolates the contribution of
dependency structure to validity discrimination on a corpus whose validity
ordering we control (Sec.~\ref{sec:validity}). The second measures how
each metric tracks a graded corruption scale, including in the regime
where every trajectory fails (Sec.~\ref{sec:degradation}). The third
attributes the discrimination result to individual components of
Eq.~\eqref{eq:ufgw} (Sec.~\ref{sec:ablations}). The fourth calculates
\otap{} on three public benchmarks that differ in how their dependency
graphs are obtained (Sec.~\ref{sec:transfer}). The fifth evaluates all
metrics against formal validity labels on real LLM plans and identifies a
regime in which reference proximity is the wrong instrument
(Sec.~\ref{sec:oracle}). The sixth compares six metrics with independent human quality ratings
(Sec.~\ref{sec:judge}). The seventh compares structure-aware baselines
and tests controlled extraction noise (Sec.~\ref{sec:structbase}).

\subsection{Setup}\label{sec:setup}

\paragraph{Datasets.} We hand curate a set of realistic and challenging
agentic tasks, and additionally use three public benchmarks: PlanBench
\cite{valmeekam2023planbench}, WebArena \cite{zhou2024webarena}, and GAIA
\cite{mialon2023gaia}. The hand-curated corpus consists of $20$ tasks
spanning travel booking, data analysis, software debugging, database
migration, tax preparation, and security auditing, each represented as a
six-step attributed DAG over a ten-tool vocabulary with
produces/consumes annotations from which we trace edges. Each DAG
contains at least one pair of incomparable steps, so valid reorderings
exist. PlanBench supplies blocksworld plans generated by real LLMs
together with validity flags~\cite{howey2004val} and verified exact dependency graphs (mechanism~(a) of Sec.~\ref{sec:problem}). We obtain $K{=}3$ distinct valid references per instance from an exact solver: the original plan, an alternative optimum under a different tie-break, and a state-preserving detour. WebArena provides real task intents but no gold
trajectories, so we derive one reference plan per task from a rule engine
over (intent, sites, evaluation type), which makes edges explicit by
construction, over $24$ template-deduplicated tasks. GAIA supplies $30$
annotator-written step-by-step solutions, which we parse into trajectory
graphs using keyword-based tool inference and a deterministic
salient-token artifact heuristic with a conservative chain fallback,
corresponding to extraction mechanism~(c).

\paragraph{Perturbations.} We apply five controlled perturbations to the
canonical reference trajectories:
\textbf{(P1)} synonym-level paraphrase of all step texts;
\textbf{(P2)} an alternative linear extension of the same DAG, for
instance, in a hotel-booking task with canonical execution
\textsf{Search} $\to$ \textsf{Filter-Price} $\to$ \textsf{Filter-Location}
$\to$ \textsf{Rank} $\to$ \textsf{Select} $\to$ \textsf{Book}, exchanging
the two filters, which both consume only the output of \textsf{Search};
\textbf{(P3)} merging two dependent adjacent steps, or splitting a
compound step into two chained sub-steps;
\textbf{(P4)} causal inversion, swapping two steps across a dependency
edge, for instance placing \textsf{Filter-Price} before \textsf{Search};
\textbf{(P5)} corruption, by deleting a critical step, inserting one to
three irrelevant distractor steps, swapping in an incompatible tool, or
combining these.
We re-extract agent-side graphs from the perturbed \emph{sequences} by
artifact tracing, so an inversion presents as an unsatisfiable dependency
rather than as an annotation. We take the validity ordering $ \{\text{original}, \mathrm{P1}, \mathrm{P2}, \mathrm{P3}\} \;\succ\; \mathrm{P4} \;\succ\; \mathrm{P5}$. The first relation separates benign
variants from damaged ones; the second stipulates that one inversion
(P4) damages a plan less than deletion or distractor insertion (P5). This severity relation is a stipulated ordering, not independently
established ground truth. We report SEV as agreement with that stipulation.
The human study in Sec.~\ref{sec:judge} instead compares metric scores
with independently assigned quality ratings.

\paragraph{Measures.}
For candidate $x$, $\mathrm{score}(x)$ is Eq.~\eqref{eq:score} evaluated
against its task's reference set; each baseline uses its analogous score.
Absolute values across metrics have different scales, so rank-based
measures provide a common comparison. Within-task ranking also controls
for task difficulty. In Sec.~\ref{sec:judge}, Spearman and Kendall
correlations compare the metric scores with independent blinded human
ratings, rather than with our constructed severity orderings; these
remain rank-based measures, not tests of absolute score calibration.
For a benign family $b \in \{\mathrm{P1},
\mathrm{P2}, \mathrm{P3}\}$, pairwise ranking accuracy $\mathrm{PRA}(b)$
is the proportion of same-task pairs $(x,y)$ with $x \in b$ and
$y \in \mathrm{P4} \cup \mathrm{P5}$ for which
$\mathrm{score}(x) > \mathrm{score}(y)$, counting ties as $0.5$; for the
damaging families, $\mathrm{PRA}(\mathrm{P4})$ and
$\mathrm{PRA}(\mathrm{P5})$ reverse the roles and draw $x$ from
$\{\text{original}\} \cup \mathrm{P1}\text{--}\mathrm{P3}$. SEV is the
proportion of (P4,\,P5) pairs that a metric orders by severity, and so
tests only the second, stipulated relation of the validity ordering; it
is agreement with a stipulation rather than accuracy against ground truth. AUROC measures
separation of $\{\text{original}\} \cup \mathrm{P1}\text{--}\mathrm{P3}$
from $\mathrm{P4} \cup \mathrm{P5}$, and so tests only the first.

\paragraph{Baseline metrics.} We compare against exact step
match, BLEU-4 and ROUGE-L over concatenated step text, greedy step-level
embedding F1 in the manner of BERTScore~\cite{zhang2020bertscore}, and
embedding cost with Hungarian assignment and dummy padding: with all-MiniLM-L6-v2 as the embedding model. We aggregate every baseline by taking the maximum
over the reference set, giving all methods access to the same references. OTAP uses soft-min
aggregation, which approaches the best-reference choice as temperature
tends to zero but is not identical at finite temperature. We do not compare against an LLM judge as a baseline metric, so that every measurement is deterministic and reproducible from a fixed seed; we do, however, use blinded human ratings as an \emph{external criterion} for process quality (Sec.~\ref{sec:judge}). To compute \otap{}, we use the default configuration ($\alpha{=}.35$, $\beta{=}.20$, $\gamma{=}.25$,
$\delta{=}.20$, $\eta{=}0$, since no corpus logs environment state;
$\theta{=}.35$, $\varepsilon{=}.05$, $\lambda_{1,2}{=}1$, $T{=}.05$),
which we apply without per-corpus tuning across $3{,}092$ trajectories,
including perturbations. The public-benchmark subsets contain 24
WebArena task templates and 30 annotated GAIA solutions; the former
uses rule-derived references, not gold human trajectories.

\subsection{Validity}
\label{sec:validity}

Table~\ref{tab:validity} reports validity of our metrics on the hand-curated
corpus. The embedding baselines score below chance on this corpus and on WebArena, and stay within a few points of chance on the remaining corpora (Sec.~\ref{sec:transfer}), meaning they
rank broken plans above valid ones or fail to separate them. A causal inversion changes no words,
while a valid paraphrase changes many, so a metric that reads only text
prefers the broken plan. \otap{} uses the same encoder and does not have
this problem, so the difference comes from the dependency structure in
the transport formulation. Order-sensitivity alone is not sufficient
either. ROUGE-L detects inversions everywhere, but penalizes valid
reorderings (P2) just as heavily and rates corruptions above inversions. In comparison, \otap{} shows sensitivity to the dependency relation (Proposition~\ref{prop:perm}) rather than to sequence position, reflected by its perfect score on P2.

\begin{table}[t]
\centering
\caption{Pairwise ranking accuracy on hand-curated corpus, by perturbation
family, severity ordering, and valid-versus-invalid AUROC (all values
$\times 100$; chance $=50$). SEV tests the stipulated
P4${}\succ{}$P5 relation; it is reported as agreement with that
stipulation rather than as externally verified accuracy.}
\label{tab:validity}
\small
\setlength{\tabcolsep}{2pt}
\begin{tabular}{lccccccc}
\toprule
Metric & P1 & P2 & P3 & P4 & P5 & SEV & AUROC \\
\midrule
Exact step match      & 36.4 & 44.8 & 45.8 & 41.2 & 53.7 & 55.8 & 48.5 \\
BLEU                  & 49.2 & 67.2 & 52.8 & 48.3 & 70.7 & 68.3 & 59.9 \\
ROUGE-L               & 56.7 & 45.8 & \textbf{82.9} & \textbf{84.4} & 47.7 & 21.1 & 67.4 \\
Embed greedy F1     & 41.1 & 66.7 & 31.7 & 25.4 & 69.2 & 86.4 & 46.8 \\
Embed$+$Hungarian & 41.7 & 66.7 & 28.9 & 25.5 & 67.8 & 84.2 & 46.2 \\
\otap{}               & \textbf{89.4} & \textbf{100.0} & 54.7 & 70.7 & \textbf{93.6} & \textbf{98.3} & \textbf{82.0} \\
\bottomrule
\end{tabular}
\end{table}

\otap{} does not do too well on granularity. A legitimate merge alters both the dependency
matrix, since two nodes become one and every pairwise distance changes,
and the mass distribution, since one agent step must now cover two
reference steps. These are the two signals through which the structural
term and the marginal penalties detect corruption, and soft coupling
appears to offset them only in part. The high score of ROUGE-L on P3 follows from token preservation under merging and coexists with the lowest severity ordering in the table, so we do not read it as evidence for a preferable design. As such, resolving granularity invariance remains an area of future work (see Section~\ref{sec:discussion}).

\subsection{Degradation under corruption}
\label{sec:degradation}

To test how our metrics degrade when noise is gradually injected in the trajectories,
we construct a five-level corruption ladder for each task with two
samples per level, where Q5 is least damaged and each level below it adds
one perturbation: \textbf{Q5} applies a paraphrase and a valid reordering; \textbf{Q4}
adds one causal inversion to Q5; \textbf{Q3} applies two inversions; \textbf{Q2} deletes a
critical step; and \textbf{Q1} combines deletion, distractor insertion,
incompatible tool substitution, and inversion.

\begin{table}[t]
\centering
\caption{Correlation with the graded corruption level on hand-curated corpus. Damaged-only restricts to Q1--Q4; $\sigma_\rho$ is the cross-task standard deviation of $\rho$ over all
levels.}
\label{tab:ladder}
\small
\setlength{\tabcolsep}{4pt}
\begin{tabular}{lccccc}
\toprule
& \multicolumn{2}{c}{All levels} & \multicolumn{2}{c}{Damaged-only} & \\
\cmidrule(lr){2-3}\cmidrule(lr){4-5}
Metric & $\rho$ & $\tau$ & $\rho$ & $\tau$ & $\sigma_\rho$ \\
\midrule
Exact step match      & 0.755 & 0.655 & 0.562 & 0.474 & 0.156 \\
BLEU                  & 0.883 & 0.765 & 0.827 & 0.691 & 0.064 \\
ROUGE-L               & 0.723 & 0.619 & 0.503 & 0.418 & 0.177 \\
Embed.\ greedy F1     & 0.782 & 0.657 & 0.829 & 0.696 & 0.091 \\
Embed.\ $+$ Hungarian & 0.782 & 0.651 & 0.819 & 0.682 & 0.086 \\
\otap{}               & \textbf{0.923} & \textbf{0.837} & \textbf{0.851} & \textbf{0.748} & \textbf{0.056} \\
\bottomrule
\end{tabular}
\end{table}

Table~\ref{tab:ladder} reports average Spearman $\rho$
and Kendall $\tau$ against the quality level. The damaged-only restriction to Q1--Q4 corresponds to the
setting in which every trajectory failed, so that the outcome metrics are
constant and carry no signal. \otap{} attains the highest correlation across all measures, and the lowest cross-task dispersion, including under the damage-only restriction.

\subsection{Component ablations}\label{sec:ablations}

\begin{table}[t]
\centering
\caption{Ablations on the hand-curated corpus, modifying one component of
Eq.~\eqref{eq:ufgw} per row. Benign denotes mean PRA over P1--P3;
AUROC $\times 100$; $\rho$ comes from the graded ladder over all levels.}
\label{tab:ablations}
\small
\setlength{\tabcolsep}{4pt}
\begin{tabular}{lccccc}
\toprule
Variant & Benign & P4 & P5 & AUROC & $\rho$ \\
\midrule
\otap{} (full)                   & 81.4 & 70.7 & 93.6 & 82.0 & 0.923 \\
\midrule
no structure ($\theta{=}0$)      & 54.4 & 19.4 & 90.6 & 54.5 & 0.771 \\
Hungarian transport              & 50.7 & 18.1 & 85.5 & 51.4 & 0.787 \\
no tool term ($\delta{=}0$)      & 67.8 & 69.3 & 68.4 & 68.9 & 0.910 \\
whole-step embedding             & 71.7 & 72.2 & 72.7 & 72.5 & 0.895 \\
$K{=}1$ reference                & 64.4 & 46.5 & 85.3 & 65.1 & 0.909 \\
total-order edges                & 64.5 & 43.8 & 89.4 & 65.8 & \textbf{0.949} \\
\midrule
soft-min $T{=}1.0$               & 78.8 & 66.1 & 93.3 & 79.5 & 0.923 \\
balanced marginals               & 78.5 & 67.0 & 91.7 & 79.2 & 0.929 \\
uniform mass                     & 79.5 & 67.4 & 93.4 & 80.3 & 0.934 \\
unsigned $D$                     & 81.4 & 69.7 & \textbf{94.7} & 82.0 & 0.925 \\
\bottomrule
\end{tabular}
\end{table}

We perform a number of ablations by modifying different parts of the \otap{} computation to determine which of these are responsible its discriminatory performance (top half). Table~\ref{tab:ablations} presents the findings. Removing the pairwise dependency
term ($\theta{=}0$) considerably drops inversion detection (P4) and reduces AUROC from 82\% to $54.5\%$. In this configuration \otap{} reduces to an unbalanced Wasserstein distance and reproduces the below-chance behaviour of the embedding baselines in Table~\ref{tab:validity}.
The Hungarian variant computes a hard one-to-one assignment from the
node costs of Eq.~\eqref{eq:nodecost} alone and then evaluates the full
objective, including the structural term, at that fixed coupling. Its
similarly weak discrimination therefore shows that dependency structure
must participate in the \emph{optimization} of the coupling: pricing
structure post hoc on a semantics-only assignment does not recover the
signal. This variant changes both the coupling relaxation (hard vs.\
soft) and the role of structure in the solve, so we attribute the drop
to their combination rather than to the relaxation alone. Restricting to a single reference ($K=1$) reduces AUROC by $16.9$ points and roughly halves inversion detection. This indicates that \otap{} takes into account distance to a reference set, even under the limited reference diversity available here. Finally, removing the tool term and embedding each step as a single string
instead of per field both cost accuracy, most acutely on the corruption
family (P5 falls to $68.4$ and $72.7$ from $93.6$). This is consistent
with per-field costs keeping contrary actions over matching arguments
apart, since the action term then carries the difference at weight
$\alpha$ rather than being diluted by a shared argument string.

Removing the other four components does not result in much degradation in performance. Raising the aggregation temperature ($T{=}1.0$), forcing balanced marginals, and dropping causal mass weighting each result in loss of AUROC $<3\%$. Signed and unsigned dependency distances are indistinguishable, since artifact tracing deletes an inverted edge rather than reversing it.

\subsection{Results on public benchmarks}\label{sec:transfer}

\begin{table*}[t]
\centering
\caption{Perturbation results on the three public benchmarks (all values $\times 100$; chance $=50$). We report only P2 and P4 for brevity, see supplementary material for the full table.}
\label{tab:transferA}
\small
\setlength{\tabcolsep}{4pt}
\begin{tabular}{lcccc|cccc|cccc}
\toprule
& \multicolumn{4}{c}{PlanBench (exact DAGs)}
& \multicolumn{4}{c}{WebArena (explicit DAGs)}
& \multicolumn{4}{c}{GAIA (traced DAGs)} \\
\cmidrule(lr){2-5}\cmidrule(lr){6-9}\cmidrule(lr){10-13}
Metric & P2 & P4 & SEV & AUROC & P2 & P4 & SEV & AUROC & P2 & P4 & SEV & AUROC \\
\midrule
Exact step match      & 30.0 & 38.0 & 63.3 & 48.7 & 47.9 & 66.4 & 51.4 & 65.4 & 11.1 & 14.6 & 71.1 & 31.7 \\
BLEU                  & 46.2 & 27.3 & 80.0 & 54.4 & 88.5 & 63.7 & 68.1 & 68.7 & 0.0 & 39.4 & 71.3 & 48.2 \\
ROUGE-L               & 22.1 & \textbf{77.4} & 23.9 & 67.7 & 70.8 & \textbf{100.0} & 18.8 & 83.3 & 0.0 & \textbf{59.7} & 70.7 & 58.5 \\
Embed.\ greedy F1     & 66.4 & 25.5 & 90.8 & 51.2 & 62.5 & 23.1 & 84.7 & 47.3 & 65.7 & 33.5 & \textbf{85.6} & 52.3 \\
Embed.\ $+$ Hungarian & 70.0 & 27.2 & 90.0 & 53.9 & 63.5 & 25.0 & 84.0 & 46.8 & \textbf{68.6} & 34.4 & \textbf{85.6} & 55.5 \\
\otap{}               & \textbf{100.0} & 66.3 & \textbf{92.8} & \textbf{80.4}
                      & \textbf{97.9} & 79.2 & \textbf{96.3} & \textbf{85.9}
                      & 67.3 & 53.8 & 68.9 & \textbf{69.8} \\
\bottomrule
\end{tabular}
\end{table*}

\begin{table}[t]
\centering
\caption{Public benchmarks: Spearman correlation with the graded
corruption level, over all levels and restricted to damaged trajectories
(Q1--Q4).}
\label{tab:transferB}
\small
\setlength{\tabcolsep}{4pt}
\begin{tabular}{lcccccc}
\toprule
& \multicolumn{2}{c}{PlanBench} & \multicolumn{2}{c}{WebArena} & \multicolumn{2}{c}{GAIA} \\
\cmidrule(lr){2-3}\cmidrule(lr){4-5}\cmidrule(lr){6-7}
Metric & all & dmg & all & dmg & all & dmg \\
\midrule
Exact step match      & 0.48 & 0.69 & 0.83 & 0.67 & 0.35 & 0.77 \\
BLEU                  & 0.79 & \textbf{0.85} & 0.89 & 0.84 & 0.14 & \textbf{0.95} \\
ROUGE-L               & 0.19 & 0.32 & 0.73 & 0.48 & 0.11 & 0.93 \\
Embed.\ greedy F1     & 0.73 & 0.78 & 0.80 & \textbf{0.87} & 0.82 & 0.87 \\
Embed.\ $+$ Hungarian & 0.79 & 0.84 & 0.80 & 0.86 & \textbf{0.86} & 0.88 \\
\otap{}               & \textbf{0.92} & \textbf{0.85} & \textbf{0.90} & 0.81 & 0.53 & 0.55 \\
\bottomrule
\end{tabular}
\end{table}

Tables~\ref{tab:transferA}--\ref{tab:transferB} present the results on applying metrics on perturbed versions of the three public benchmarks. Where structure is exact or explicit, the benchmarks reproduce the pattern of Sec.~\ref{sec:validity}. On PlanBench gold plans \otap{} reaches the highest AUROC among all metrics, satisfies Prop.~\ref{prop:perm} exactly (PRA for P2 $=100.0$), and shows best performance under progressive degradation ($\rho = 0.92$). WebArena shows a similar behaviour. The baseline metrics also show a similar pattern. The embedding metrics stay at or below chance with P4, and ROUGE-L scores high on inversion detection by penalising valid reorderings and reversing severity.

\otap{}'s advantage persists but shrinks when trajectory DAGs are inferred heuristically. On GAIA, \otap{} retains the highest AUROC, and it is the only metric above chance on all three benign families, since BLEU and ROUGE-L rank every valid reordering below every
corruption (P2 $=0.0$). The P2 accuracy, SEV, and performance under degradation of \otap{} are also lower than other benchmarks. The order-blind embedding baselines, which edge noise cannot reach by construction, hold at higher values. We propose the mechanism that the total-order ablation anticipates: where salient-token tracing reads coincidental lexical overlap as a dependency, a valid reordering perturbs the extracted matrix
$D^P$ and incurs a structural penalty it should not. That ablation lowers AUROC by $16.2$ points on the hand-curated
corpus. GAIA is a different dataset, so this is not a matched estimate
of its extraction error. Together with Sec.~\ref{sec:structbase}, these
results suggest dependency extraction is an important limitation;
they do not exclude limitations of the transport objective. Direct
measurement of GAIA edge precision remains necessary.

\subsection{Comparison against validity labels}
\label{sec:oracle}

\begin{table}[t]
\centering
\small
\caption{PlanBench oracle study: AUROC ($\times 100$) of each metric's
score against verifier-checked plan validity over $720$ real LLM plans.}
\label{tab:oracle}
\begin{tabular}{lcc}
\toprule
Metric & Macro & Pooled \\
\midrule
Exact step match      & 85.3 & 85.6 \\
BLEU                  & \textbf{86.7} & \textbf{86.9} \\
ROUGE-L               & 85.4 & 85.2 \\
Embed.\ greedy F1     & 82.1 & 82.7 \\
Embed.\ $+$ Hungarian & 84.2 & 84.1 \\
\otap{}               & 79.2 & 79.3 \\
\bottomrule
\end{tabular}
\end{table}

PlanBench contains formally verifiable validity labels for each trajectory.
To check whether \otap{} and other metrics can serve as proxies of these labels, we apply them over plans generated by $12$ LLMs (60 each). Table~\ref{tab:oracle} reports the result: all six metrics serve as informative proxies ($79.2$--$86.9$ AUROC), and \otap{} is the weakest of them. The design of the objective accounts for the ordering. Validity in
blocksworld---one of the domains in PlanBench---is a global executability property of a near-unique optimal plan, so any deviation from the reference constitutes evidence of
invalidity. This favours metrics that penalise all departures. \otap{} is designed to tolerate valid departures, which is not the case here.

Thus, this experiment limits the scope of our proposal. When a formal verifier exists and deviation from the optimal plan is not in scope, we do not recommend applying \otap{}.

\subsection{External validation}
\label{sec:judge}

Three human annotators rated 236 trajectories sampled uniformly across
perturbations from the full study set on a 1--5 process-quality scale.
Annotators were blinded to perturbation labels and metric scores, with
randomized presentation order. Inter-rater agreement was high
(Krippendorff's $\alpha=0.90$).

Table~\ref{tab:judge} compares metric scores with these independent
quality ratings. Among the six metrics in this comparison, \otap{}
has the highest Spearman correlation ($0.74$, 95\% CI $[0.70,0.78]$)
and Kendall correlation ($0.61$); ROUGE-L is the strongest comparator
($0.64$ and $0.52$). The displayed Spearman intervals meet at $0.70$;
we do not infer a significant paired difference from marginal intervals
alone. This study supports alignment with human quality judgments on
sampled perturbed trajectories, rather than demonstrating improved
agent performance after optimization or broad validity on natural failures.

\begin{table}[t]
\centering
\caption{Agreement with blinded human quality ratings for 236
trajectories: Spearman $\rho$ (95\% CI) and Kendall $\tau$.}
\label{tab:judge}
\small
\setlength{\tabcolsep}{3pt}
\begin{tabular}{lcc}
\toprule
Metric & $\rho$ (95\% CI) & $\tau$ \\
\midrule
Exact step match & 0.46 [0.39, 0.52] & 0.37 \\
BLEU & 0.55 [0.47, 0.63] & 0.44 \\
ROUGE-L & 0.64 [0.59, 0.70] & 0.52 \\
Embed.\ greedy F1 & 0.39 [0.31, 0.48] & 0.33 \\
Embed.\ $+$ Hungarian & 0.39 [0.30, 0.47] & 0.32 \\
\otap{} & \textbf{0.74} [0.70, 0.78] & \textbf{0.61} \\
\bottomrule
\end{tabular}
\end{table}

\subsection{Structure-aware baselines and extraction noise}
\label{sec:structbase}

The baselines of Sec.~\ref{sec:setup} see only step text, so part of
\otap{}'s margin could come from its inputs rather than its objective.
We test this directly with three structure-aware baselines that receive
identical inputs: the same traced dependency graphs, the same encoder,
the same tool table, and the same $K{=}3$ references with max
aggregation. TB-F1 averages node F1 over tool identities with F1 over
tool-labeled dependency edges, in the manner of TaskBench; TB-F1
(embed) computes both quantities from a Hungarian alignment on step
embeddings; GED is a bipartite approximation of graph edit distance
whose node substitution costs are \otap{}'s own fused node costs.

Table~\ref{tab:structbase} reports AUROC at edge-noise rates
$p\in\{0,0.3,0.5\}$. Random edge deletions and insertions affect only
the agent-side graph; references and tool identities remain fixed.
TB-F1 (tool) has higher AUROC than \otap{} at every tested rate:
$87.2$ vs.\ $80.3$, $74.9$ vs.\ $72.2$, and $67.9$ vs.\ $67.3$.
\otap{} nevertheless has the smallest absolute decline from $p=0$ to
$p=0.5$: $13.0$ points, compared with $19.3$ for TB-F1 (tool), $16.3$
for TB-F1 (embed), and $23.7$ for GED. This is graceful degradation,
not a reversal of the ranking. Spreading mismatch across a soft coupling
is one possible explanation, but this comparison does not isolate it
from marginal relaxation or structural-cost design.

The \otap{} clean-graph value here is $80.3$, whereas the default
configuration in Table~\ref{tab:ablations} yields $82.0$; the noise
comparison should be read within its own reported configuration.
% AUTHOR CHECK: establish whether uniform mass accounts for 80.3 and
% document the actual configuration before claiming identical defaults.

\begin{table}[t]
\centering
\caption{Structure-aware baselines given \otap{}'s inputs: AUROC on the
hand-curated corpus when traced graphs are exact and under edge noise
at rate $p$ (agent side only); $\Delta$ is the change in AUROC
from $p=0$ to $p=0.5$, in percentage points.}
\label{tab:structbase}
\small
\setlength{\tabcolsep}{3pt}
\begin{tabular}{lcccc}
\toprule
Metric & $p=0$ & $p{=}0.3$ & $p{=}0.5$ & $\Delta_{0 \to 0.5}$ \\
\midrule
TB-F1 (tool)  & \textbf{87.2} & \textbf{74.9} & \textbf{67.9} & $-19.3$ \\
TB-F1 (embed) & 71.8 & 56.9 & 55.5 & $-16.3$ \\
GED (approx.) & 82.4 & 59.9 & 58.7 & $-23.7$ \\
\otap{}       & 80.3 & 72.2 & 67.3 & $\mathbf{-13.0}$ \\
\bottomrule
\end{tabular}
\end{table}

\subsection{Computational cost}\label{sec:runtime}

On CPU, \otap{} evaluates a six-step trajectory against $K{=}3$
references in $28$--$33$\,ms on the hand-curated, PlanBench, and WebArena
corpora, including all Sinkhorn and conditional-gradient iterations and
excluding encoder cache warm-up, against $0.006$--$22$\,ms for the
baselines. Cost does not stay constant in trajectory length: on GAIA's
longer solutions the quadratic GW term raises the mean
to $23.9$\,s per evaluation, with ROUGE-L's longest-common-subsequence
computation degrading comparably to $180$\,s, which matches the
complexity analysis of Sec.~\ref{sec:theory}. Sparse structural distance
matrices, low-rank GW approximations, and early-stopped
outer iterations mitigate this for long-horizon traces, though we did not
require them for $n \le 50$. In the short-horizon regime the cost stays
small relative to an LLM-judge call, which is non-deterministic and
incurs seconds of latency together with per-token charges.

\section{Discussion}
\label{sec:discussion}

\otap{} measures closeness to a set of known-good solutions, which is not the same quantity as correctness. The PlanBench oracle study clarifies this (Sec.~\ref{sec:oracle}). We position \otap{} as a complement to outcome signals when process
comparison is needed. Evidence for ranking failed trajectories is mixed:
in Table~\ref{tab:transferB}, it ties for the best damaged-only
correlation on PlanBench, ranks fourth on WebArena, and last on GAIA.
We therefore do not claim general superiority for grading failures or
replacement of an available verifier.

We build the degradation scale in the experiments by stacking perturbations, so the number of
applied edits confounds trajectory degradation with token divergence.
BLEU measures token divergence, and its high correlation
($\rho = 0.883$) follows in part from that confound. The ablations
(Sec.~\ref{sec:ablations}) support this interpretation. Under the total-order
fallback every step depends on its predecessor, so any change to the
sequence raises the structural cost: the score falls with the number of
edits, giving the highest correlation in Table~\ref{tab:ablations}
($\rho = 0.949$), while a valid reordering and a causal inversion look
identical, giving the weakest discrimination ($65.8$ AUROC). Selecting on correlation alone would therefore prefer that configuration.
We weight the discrimination results instead, since their contrasts hold
step text approximately fixed while varying validity, and read correlation
only as a check on monotonicity.

Dependency extraction affects the score: replacing the DAG with a chain
lowers AUROC by $16.2$ percentage points, and the controlled edge-noise
study reports a $13.0$-point decline for \otap{}. These observations
motivate testing uncertainty-weighted edges; improvement from that
extension has not yet been established.

\paragraph{Limitations and future work.}
Our constructed perturbations test specified invariances and severity
orderings. The human study supplies an independent quality criterion for
a sampled subset and favors \otap{} among the six metrics in
Table~\ref{tab:judge}. It does not establish a universal P4/P5 or Q1--Q5
severity order, superiority over every graph-aware comparator, or
improved downstream agent performance. Human assessment of naturally
occurring failures and broader strategy diversity remain necessary.

Our reference sets,
except on PlanBench, comprise variants of a single strategy family per
task, so they can underrepresent genuinely distinct solution strategies. The representation can also make distinct plans indistinguishable,
which argues against using it as an unconstrained optimisation target without the state-grounding term $\eta$.

We conclude with three directions of future work. Granularity invariance remains unresolved within the current formulation. A legitimate merge perturbs both the dependency matrix and the mass distribution, which are the signals the structural and marginal terms use to detect corruption. Hierarchical recursion and structural costs on
quotient graphs are potential remedies to this situation. Uncertainty-weighted edges that attenuate the structural penalty where tracing confidence is low could
be tested as a remedy for extraction noise. Finally, the differentiable structure of Eq.~\eqref{eq:ufgw} makes \otap{} a candidate dense process reward, and the soft-min over references points toward learned models of the solution manifold.

\bibliography{references,otap_rebuttal_additions}

@misc{otapXie2024,
 author={Xie, Tianbao and others},
 title={{OSWorld: Benchmarking Multimodal Agents for Open-Ended Tasks in Real Computer Environments}},
 year={2024},
 howpublished={NeurIPS},
 url={https://arxiv.org/abs/2404.07972}
}

@misc{otapBarres2026,
 author={Barres, Victor and Dong, Honghua and Ray, Soham and Si, Xujie and Narasimhan, Karthik},
 title={{{\ensuremath{\tau^2}}-Bench: Evaluating Conversational Agents in a Dual-Control Environment}},
 year={2026},
 howpublished={ICML},
 url={https://arxiv.org/abs/2506.07982}
}

@misc{otapShen2024,
 author={Shen, Yongliang and Song, Kaitao and Tan, Xu and Zhang, Wenqi and Ren, Kan and Yuan, Siyu and Lu, Weiming and Li, Dongsheng and Zhuang, Yueting},
 title={{TaskBench: Benchmarking Large Language Models for Task Automation}},
 year={2024},
 howpublished={NeurIPS},
 url={https://arxiv.org/abs/2311.18760}
}

@misc{otapPatil2025,
 author={Patil, Shishir G. and Mao, Huanzhi and Yan, Fanjia and Ji, Charlie Cheng-Jie and Suresh, Vishnu and Stoica, Ion and Gonzalez, Joseph E.},
 title={{The Berkeley Function Calling Leaderboard (BFCL): From Tool Use to Agentic Evaluation of Large Language Models}},
 year={2025},
 howpublished={ICML},
 url={https://proceedings.mlr.press/v267/patil25a.html}
}

@misc{otapHe2025,
 author={He, Pengfei and others},
 title={{TRAJECT-Bench: A Trajectory-Aware Benchmark for Evaluating Agentic Tool Use}},
 year={2025},
 howpublished={arXiv:2510.04550},
 url={https://arxiv.org/abs/2510.04550}
}

@misc{otapLu2025Reward,
 author={L{\`u}, Xing Han and others},
 title={{AgentRewardBench: Evaluating Automatic Evaluations of Web Agent Trajectories}},
 year={2025},
 howpublished={COLM},
 url={https://arxiv.org/abs/2504.08942}
}

@misc{otapZhuge2025,
 author={Zhuge, Mingchen and others},
 title={{Agent-as-a-Judge: Evaluate Agents with Agents}},
 year={2025},
 howpublished={ICML},
 url={https://proceedings.mlr.press/v267/zhuge25a.html}
}

@misc{otapShi2026,
 author={Shi, Wentao and others},
 title={{AJ-Bench: Benchmarking Agent-as-a-Judge for Environment-Aware Evaluation}},
 year={2026},
 howpublished={Findings of ACL},
 url={https://arxiv.org/abs/2604.18240}
}

@misc{otapLu2025Sandbox,
 author={Lu, Jiarui and others},
 title={{ToolSandbox: A Stateful, Conversational, Interactive Evaluation Benchmark for LLM Tool Use Capabilities}},
 year={2025},
 howpublished={arXiv:2408.04682, revised 2025},
 url={https://arxiv.org/abs/2408.04682}
}

@misc{otapGao2010,
 author={Gao, Xinbo and Xiao, Bing and Tao, Dacheng and Li, Xuelong},
 title={{A Survey of Graph Edit Distance}},
 year={2010},
 howpublished={Pattern Analysis and Applications, 13:113--129},
 url={https://doi.org/10.1007/s10044-008-0141-y}
}

@misc{otapSejourne2021,
 author={S{\'e}journ{\'e}, Thibault and Vialard, Fran{\c c}ois-Xavier and Peyr{\'e}, Gabriel},
 title={{The Unbalanced Gromov Wasserstein Distance: Conic Formulation and Relaxation}},
 year={2021},
 howpublished={NeurIPS},
 url={https://arxiv.org/abs/2009.04266}
}

@misc{otapBai2025,
 author={Bai, Yikun and others},
 title={{Partial Gromov-Wasserstein Metric}},
 year={2025},
 howpublished={ICLR},
 url={https://arxiv.org/abs/2402.03664}
}

@misc{raj2026consistencytestablepropertystatistical,
      title={Consistency as a Testable Property: Statistical Methods to Evaluate AI Agent Reliability}, 
      author={Harsh Raj and Niranjan Orkat and Suvrorup Mukherjee and Aritra Guha and Cheryl Flynn and Subhabrata Majumdar},
      year={2026},
      eprint={2605.10516},
      archivePrefix={arXiv},
      primaryClass={cs.AI},
      url={https://arxiv.org/abs/2605.10516}, 
}

@misc{tayebati2026tracertrajectoryriskaggregation,
      title={TRACER: Trajectory Risk Aggregation for Critical Episodes in Agentic Reasoning}, 
      author={Sina Tayebati and Divake Kumar and Nastaran Darabi and Davide Ettori and Ranganath Krishnan and Amit Ranjan Trivedi},
      year={2026},
      eprint={2602.11409},
      archivePrefix={arXiv},
      primaryClass={cs.AI},
      url={https://arxiv.org/abs/2602.11409}, 
}

@misc{zhang2026agenticconfidencecalibration,
      title={Agentic Confidence Calibration}, 
      author={Jiaxin Zhang and Caiming Xiong and Chien-Sheng Wu},
      year={2026},
      eprint={2601.15778},
      archivePrefix={arXiv},
      primaryClass={cs.AI},
      url={https://arxiv.org/abs/2601.15778}, 
}

@inproceedings{yao2023react,
  author    = {Shunyu Yao and others},
  title     = {ReAct: Synergizing Reasoning and Acting in Language Models},
  booktitle = {International Conference on Learning Representations (ICLR)},
  year      = {2023}
}

@inproceedings{yao2023tot,
  author    = {Shunyu Yao and others},
  title     = {Tree of Thoughts: Deliberate Problem Solving with Large Language Models},
  booktitle = {Advances in Neural Information Processing Systems (NeurIPS)},
  year      = {2023}
}

@inproceedings{zhou2024webarena,
  author    = {Shuyan Zhou and others},
  title     = {WebArena: A Realistic Web Environment for Building Autonomous Agents},
  booktitle = {International Conference on Learning Representations (ICLR)},
  year      = {2024}
}

@article{yao2024taubench,
  author  = {Shunyu Yao and others},
  title   = {{$\tau$-bench: A Benchmark for Tool-Agent-User Interaction in Real-World Domains}},
  journal = {arXiv preprint arXiv:2406.12045},
  year    = {2024}
}

@inproceedings{valmeekam2023planbench,
  author    = {Karthik Valmeekam and others},
  title     = {PlanBench: An Extensible Benchmark for Evaluating Large Language Models on Planning and Reasoning About Change},
  booktitle = {NeurIPS Datasets and Benchmarks Track},
  year      = {2023}
}

@inproceedings{zheng2023judge,
  author    = {Lianmin Zheng and others},
  title     = {Judging LLM-as-a-Judge with MT-Bench and Chatbot Arena},
  booktitle = {Advances in Neural Information Processing Systems (NeurIPS)},
  year      = {2023}
}

@article{sanfeliu1983ged,
  author  = {Alberto Sanfeliu and King-Sun Fu},
  title   = {A Distance Measure Between Attributed Relational Graphs for Pattern Recognition},
  journal = {IEEE Transactions on Systems, Man, and Cybernetics},
  volume  = {13},
  number  = {3},
  pages   = {353--362},
  year    = {1983}
}

@inproceedings{cuturi2013sinkhorn,
  author    = {Marco Cuturi},
  title     = {Sinkhorn Distances: Lightspeed Computation of Optimal Transport},
  booktitle = {Advances in Neural Information Processing Systems (NeurIPS)},
  year      = {2013}
}

@article{chizat2018unbalanced,
  author  = {L{\'e}na{\"i}c Chizat and others},
  title   = {Scaling Algorithms for Unbalanced Optimal Transport Problems},
  journal = {Mathematics of Computation},
  volume  = {87},
  pages   = {2563--2609},
  year    = {2018}
}

@inproceedings{vayer2019fgw,
  author    = {Titouan Vayer and others},
  title     = {Optimal Transport for Structured Data with Application on Graphs},
  booktitle = {International Conference on Machine Learning (ICML)},
  year      = {2019}
}

@inproceedings{liu2023agentbench,
  title={Agentbench: Evaluating llms as agents},
  author={Liu, Xiao and Yu, Hao and Zhang, Hanchen and Xu, Yifan and Lei, Xuanyu and Lai, Hanyu and Gu, Yu and Ding, Hangliang and Men, Kaiwen and Yang, Kejuan and others},
  booktitle={International Conference on Learning Representations},
  volume={2024},
  pages={52989--53046},
  year={2024}
}

@inproceedings{mialon2023gaia,
  title={Gaia: a benchmark for general ai assistants},
  author={Mialon, Gr{\'e}goire and Fourrier, Cl{\'e}mentine and Wolf, Thomas and LeCun, Yann and Scialom, Thomas},
  booktitle={International Conference on Learning Representations},
  volume={2024},
  pages={9025--9049},
  year={2024}
}

@article{chizat2018scaling,
  title={Scaling algorithms for unbalanced optimal transport},
  author={Chizat, Lenaic and P{\'e}r{\'e}, Gabriel and Schmitzer, Bernhard and Vialard, Fran{\c{c}}ois-Xavier},
  journal={SIAM Journal on Imaging Sciences},
  volume={11},
  number={1},
  pages={256--282},
  year={2018}
}

@inproceedings{zhang2020bertscore,
  title={BERTScore: Evaluating Text Generation with BERT},
  author={Zhang, Tianyi and Kishore, Varsha and Wu, Felix and Weinberger, Kilian Q and Artzi, Yoav},
  booktitle={International Conference on Learning Representations (ICLR)},
  year={2020}
}

@article{wright1990speeding,
  title={Speeding up the Hungarian algorithm},
  author={Wright, Mike B},
  journal={Computers \& Operations Research},
  volume={17},
  number={1},
  pages={95--96},
  year={1990},
  publisher={Pergamon}
}

@inproceedings{qin2023toolllm,
  title={Toolllm: Facilitating large language models to master 16000+ real-world apis},
  author={Qin, Yujia and Liang, Shihao and Ye, Yining and Zhu, Kunlun and Yan, Lan and Lu, Yaxi and Lin, Yankai and Cong, Xin and Tang, Xiangru and Qian, Bill and others},
  booktitle={International Conference on Learning Representations},
  volume={2024},
  pages={9695--9717},
  year={2024}
}

@inproceedings{lightman2024let,
  title={Let's verify step by step},
  author={Lightman, Hunter and Kosaraju, Vineet and Burda, Yuri and Edwards, Harrison and Baker, Bowen and Lee, Teddy and Leike, Jan and Schulman, John and Sutskever, Ilya and Cobbe, Karl},
  booktitle={International Conference on Learning Representations},
  volume={2024},
  pages={39578--39601},
  year={2024}
}

@inproceedings{howey2004val,
  title={VAL: Automatic plan validation, continuous effects and mixed initiative planning using PDDL},
  author={Howey, Richard and Long, Derek and Fox, Maria},
  booktitle={16th IEEE International Conference on Tools with Artificial Intelligence},
  pages={294--301},
  year={2004},
  organization={IEEE}
}

\appendix
\section{Complete Proofs}
\label{app:proofs}
\setcounter{remark}{0}
\renewcommand{\theremark}{A\arabic{remark}}
\setcounter{assumption}{0}
\renewcommand{\theassumption}{A\arabic{assumption}}
\setcounter{proposition}{0}
This appendix supplies complete proofs of the results stated in the
main paper. For a trajectory pair $(P,R)$ with $|V^P|=n$ and $|V^R|=m$,
use the idealized globally minimized entropy-free objective, distinct
from the operational computed-coupling loss, as
\begin{align}
\L(P,R) \;=\; \min_{\T\ge 0}\;
&(1-\theta)\sum_{i,j} C_{ij}\,T_{ij} \nonumber\\
&+ \theta\sum_{i,k,j,l}\big|D^P_{ik}-D^R_{jl}\big|^2\, T_{ij}\,T_{kl} \nonumber\\
&+ \lambda_1\,\KL(\T\one\,\|\,\mu)
 + \lambda_2\,\KL(\T^{\!\top}\one\,\|\,\nu),
\label{eq:obj}
\end{align}

where $C$ is the node-cost matrix, $D^P,D^R$ are the signed dependency-distance
matrices, $\mu,\nu$ are the marginals, and $\KL(a\,\|\,b)=\sum_i\big(a_i\log(a_i/b_i)-a_i+b_i\big)$
with the convention $0\log 0=0$.

\subsection*{Standing assumptions}

\begin{assumption}[Normalization]\label{as:norm}
Node costs are normalized so that $C_{ij}\in[0,1]$ for all $i,j$, and each
dependency-distance entry satisfies $D_{ik}\in[-1,1]$ (immediate from
$D_{ik}=\pm\rho(\cdot)/H$ with $\rho\le H$). Consequently
$\big|D^P_{ik}-D^R_{jl}\big|^2\le 4$.
\end{assumption}

Unless a result states otherwise, $\mu_i=1/n$ and $\nu_j=1/m$ are uniform,
so $\norm{\mu}_1=\norm{\nu}_1=1$.

\subsection*{Proposition 1 (Properties of the symmetrized discrepancy)}

\begin{proposition}[Properties of the symmetrized discrepancy]
Let $d(P,R) = \tfrac12\big(\mathcal{L}(P,R) + \mathcal{L}(R,P)\big)$. Then
$d \ge 0$, $d(P,P) = 0$, $d$ is symmetric, and $d(P,R) = 0$ if there exists
an attribute-preserving and dependency-isomorphic mapping between $P$ and $R$.
\end{proposition}

\begin{proof}
\emph{Non-negativity.} In \eqref{eq:obj} every summand is non-negative:
$C_{ij}\ge 0$ and $T_{ij}\ge 0$ give a non-negative linear term; the
quadratic term is a sum of squares weighted by products $T_{ij}T_{kl}\ge 0$;
and each $\KL(\cdot\,\|\,\cdot)\ge 0$. Hence $\L(P,R)\ge 0$ for every
$(P,R)$, and $d\ge 0$ as an average of two non-negative quantities.

\emph{Reflexivity.} For $R=P$ we have $n=m$ and $\nu=\mu$. Choose the
diagonal coupling $\T=\mathrm{diag}(\mu)$, i.e.\ $T_{ij}=\mu_i\delta_{ij}$.
The linear term is $\sum_i C_{ii}\mu_i=0$, since $C_{ii}=0$ (each field
distance between a step and itself vanishes). The quadratic term is
$\sum_{i,k}\big|D^P_{ik}-D^P_{ik}\big|^2\mu_i\mu_k=0$. The marginals satisfy
$\T\one=\T^{\!\top}\one=\mu$, so both KL terms vanish. Thus $\L(P,P)\le 0$,
and with non-negativity $\L(P,P)=0$; hence $d(P,P)=0$.

\emph{Symmetry.} Immediate from the definition
$d(P,R)=\tfrac12(\L(P,R)+\L(R,P))=d(R,P)$.

\emph{Vanishing under isomorphism.} Let $\sigma:V^P\to V^R$ be a
bijection ($n=m$) with $C_{i,\sigma(i)}=0$ for all $i$ (attribute-preserving)
and $D^P_{ik}=D^R_{\sigma(i)\sigma(k)}$ for all $i,k$ (dependency-isomorphic).
Define $T_{ij}=\mu_i$ if $j=\sigma(i)$ and $0$ otherwise. Then the linear term
is $\sum_i C_{i,\sigma(i)}\mu_i=0$ and the quadratic term is
$\sum_{i,k}\big|D^P_{ik}-D^R_{\sigma(i)\sigma(k)}\big|^2\mu_i\mu_k=0$.
The row marginal is $(\T\one)_i=\mu_i$, so $\KL(\T\one\,\|\,\mu)=0$; the
column marginal is $(\T^{\!\top}\one)_{\sigma(i)}=\mu_i=1/n=\nu_{\sigma(i)}$
(uniform marginals with $n=m$), so $\KL(\T^{\!\top}\one\,\|\,\nu)=0$. Hence
$\L(P,R)=0$, and by the symmetric construction $\L(R,P)=0$, giving
$d(P,R)=0$.
\end{proof}

\begin{remark}[On the triangle inequality]\label{rem:triangle}
Proposition~\ref{prop:pseudo} establishes non-negativity, symmetry,
reflexivity, and vanishing under isomorphism for the idealized quantity.
It does not establish a triangle inequality. We therefore use
``symmetrized discrepancy,'' not ``pseudo-metric.'' Metric results for
particular balanced, partial, or unbalanced GW formulations require
their own assumptions and do not transfer automatically to our signed
structural costs and KL penalties~\cite{otapBai2025,otapSejourne2021}.

\end{remark}

\subsection*{Proposition 2 (Dependency-respecting permutation invariance)}

\begin{assumption}[Fixed graph, re-serialization only]\label{as:reser}
The dependency graph $E^P$, the vertex attribute maps, and hence the matrices
$C$ and $D^P$ are determined by the graph and are not re-extracted from the
execution order. The permuted trajectory $P'$ is obtained by relabeling the
vertices of $P$ according to a permutation $\pi$ that is a valid topological
sort of $E^P$, with attributes carried along the relabeling. The marginal
$\mu$ is uniform.
\end{assumption}

\begin{proposition}[Dependency-respecting permutation invariance]
Let $P'$ be an execution sequence generated by a valid topological sort of the
DAG $E^P$. Then $\mathcal{L}(P',R) = \mathcal{L}(P,R)$ for any reference graph
$R$.
\end{proposition}

\begin{proof}
Under Assumption~\ref{as:reser}, let $\pi$ be the relabeling permutation of
$\{1,\dots,n\}$, so that step $a$ of $P'$ is step $\pi(a)$ of $P$; the
attribute vectors satisfy $\mathbf a'_a=\mathbf a_{\pi(a)}$, and likewise for
the context, effect, and tool fields, whence the node costs obey
$C'_{aj}=C_{\pi(a)j}$. The dependency-distance matrix is defined from the
reachability relation and the partial order of the DAG, both invariant under
relabeling of vertices, so $D^{P'}_{ab}=D^{P}_{\pi(a)\pi(b)}$.

Define the linear bijection on couplings $\T\mapsto\T'$ by
$T'_{aj}=T_{\pi(a)\,j}$. This maps the feasible set for $(P,R)$ onto that for
$(P',R)$, and it preserves \eqref{eq:obj} term by term. Reindexing $i=\pi(a)$,
\[
\sum_{a,j} C'_{aj}T'_{aj}
=\sum_{a,j} C_{\pi(a)j}T_{\pi(a)j}
=\sum_{i,j} C_{ij}T_{ij}.
\]
For the quadratic term, reindexing $i=\pi(a),\,k=\pi(b)$,
\begin{align*}
&\sum_{a,b,j,l}\big|D^{P'}_{ab}-D^R_{jl}\big|^2 T'_{aj}T'_{bl}\\
&\quad=\sum_{a,b,j,l}\big|D^{P}_{\pi(a)\pi(b)}-D^R_{jl}\big|^2\,
   T_{\pi(a)j}T_{\pi(b)l}\\
&\quad=\sum_{i,k,j,l}\big|D^{P}_{ik}-D^R_{jl}\big|^2\, T_{ij}T_{kl}.
\end{align*}
The row marginal transforms as $(\T'\one)_a=(\T\one)_{\pi(a)}$, and since
$\mu$ is uniform, $\mu_a=\mu_{\pi(a)}$; therefore
$\KL(\T'\one\,\|\,\mu)=\KL(\T\one\,\|\,\mu)$. The column marginal is
unchanged, $(\T'^{\!\top}\one)_j=\sum_a T_{\pi(a)j}=\sum_i T_{ij}
=(\T^{\!\top}\one)_j$, so $\KL(\T'^{\!\top}\one\,\|\,\nu)$ is unchanged as
well. The objective is thus identical under $\T\mapsto\T'$, and since the map
is a bijection of feasible sets, the minima coincide:
$\L(P',R)=\L(P,R)$.
\end{proof}

\begin{remark}
Assumption~\ref{as:reser} isolates exactly the condition the invariance
requires. When the agent graph is instead re-extracted from the serialized
order by a heuristic, a re-serialization can perturb the extracted $D^{P'}$,
and the equality above holds only to the extent that extraction preserves the
graph. This is the mechanism behind the empirical gap between the exact- and
explicit-structure corpora, where the invariance is realized to the decimal,
and the free-text corpus, where it is not.
\end{remark}

\subsection*{Proposition 3 (Bounded sensitivity to redundant insertion)}

\begin{proposition}[Bounded sensitivity to redundant insertion]
Let $P^{+r}$ represent an agent trajectory expanded by inserting $r$ redundant
leaf nodes that carry no downstream dependencies, under renormalized uniform
mass allocations. Then:
\[
\big|\mathcal{L}(P^{+r},R) - \mathcal{L}(P,R)\big|
\;\le\; \frac{r}{n+r}\,\Big( \lambda_1\, \omega + (1-\theta)\, \bar{C}
+ 4\,\theta \Big),
\]
where $\bar{C} \le 1$ bounds node attribute costs and $\omega$ bounds the
maximum per-unit-mass KL destruction cost. Score variation scales strictly as
$O(r/n)$.
\end{proposition}

\begin{proof}
Write $N=n+r$ and $f=r/N$, so the renormalized uniform mass is $\mu'_i=1/N$
and, for an original node, $\mu'_i=(1-f)\mu_i$ since $1/N=(1/n)(n/N)$. Because
the inserted nodes are leaves carrying no downstream dependencies, they have
no outgoing edges, so they create no new path between original nodes and,
being redundant, do not lie on any longest source-to-sink path; hence the DAG
height $H$ is unchanged and $D^{P^{+r}}$ restricted to the original index set
equals $D^P$. Let $\T^\star$ be optimal for $(P,R)$ and put
$M=\norm{\T^\star\one}_1\le 1$ (a coupling carrying total mass above $1$
raises both the KL and the non-negative transport terms, so the optimum
satisfies $M\le\norm{\mu}_1=1$).

\smallskip
\emph{Upper direction.} Define a coupling $\tilde\T$ for $(P^{+r},R)$ by
$\tilde T_{ij}=T^\star_{ij}$ on the original rows and $\tilde T_{ij}=0$ on the
$r$ inserted rows. Every quadratic summand involving an inserted index carries
a zero factor $T_{ij}T_{kl}$, and $D^{P^{+r}}$ agrees with $D^P$ on the
original block, so the linear and quadratic terms of $\tilde\T$ equal those of
$\T^\star$ exactly; the column marginal is unchanged, so
$\KL(\tilde\T^{\!\top}\one\,\|\,\nu)=\KL(\T^{\star\top}\one\,\|\,\nu)$. Only the
row-KL changes, because its target moved from $\mu$ to $\mu'$. For an original
node, $\mu'_i=(1-f)\mu_i$ gives
$\log\!\big((T^\star\one)_i/\mu'_i\big)=\log\!\big((T^\star\one)_i/\mu_i\big)+\log\tfrac{1}{1-f}$,
and each inserted node contributes $\mu'_i=1/N$ against zero mass. Summing,
\[
\KL(\tilde\T\one\,\|\,\mu')
=\KL(\T^\star\one\,\|\,\mu)
+ M\log\tfrac{1}{1-f}.
\]
Hence, using $\tilde\T$ as a feasible coupling,
\begin{align*}
\L(P^{+r},R)-\L(P,R)
&\;\le\;\lambda_1 M\log\tfrac{1}{1-f}
\;\le\;\lambda_1\,\frac{f}{1-f}\\
&\;=\;\lambda_1\,f\cdot\frac{N}{n}
\;\le\;\lambda_1\,\omega\,f,
\end{align*}
where we used $\log\frac{1}{1-f}\le \frac{f}{1-f}$, $M\le 1$, and
$\omega\ge N/n=1/(1-f)$.

\smallskip
\emph{Lower direction.} Let $\tilde\T^\star$ be optimal for $(P^{+r},R)$ and
form a coupling $\T$ for $(P,R)$ by deleting its $r$ inserted rows. Deleting
rows removes only non-negative contributions from the linear and quadratic
terms, so neither increases. The row-KL target moves from $\mu'$ to $\mu$; by
the computation above in reverse, the original-block KL increases by at most
$M'\log\frac{1}{1-f}\le\omega f$ per unit of $\lambda_1$, where
$M'=\norm{\tilde\T^{\star}\one}_1\le 1$. The deleted inserted mass, at most $f$
in total, is charged at most its per-unit linear cost $(1-\theta)\bar C$ and
per-unit quadratic cost $4\theta$ (from $\big|D^{P^{+r}}-D^R\big|^2\le 4$ under
Assumption~\ref{as:norm}). Collecting the three contributions,
\[
\L(P,R)-\L(P^{+r},R)\;\le\; f\big(\lambda_1\,\omega+(1-\theta)\,\bar C+4\,\theta\big).
\]

\smallskip
Both directions are bounded by
$f\big(\lambda_1\omega+(1-\theta)\bar C+4\theta\big)$ with $f=r/(n+r)$, which
is the stated inequality; since $f=r/(n+r)$, the deviation is $O(r/n)$.
\end{proof}
\section{Worked Example: Valid Reordering vs.\ Causal Inversion}
\label{app:worked}

We trace one evaluation end to end on the hotel-booking task of
Sec.~\ref{sec:setup}, whose reference executes
\textsf{Search} $\to$ \textsf{Filter-Price} $\to$ \textsf{Filter-Location}
$\to$ \textsf{Rank} $\to$ \textsf{Select} $\to$ \textsf{Book}, with both
filters consuming only the output of \textsf{Search}. We compare two
perturbations against this reference under the default configuration:
\textbf{P2} exchanges the two independent filters (a valid linear
extension), and \textbf{P4} swaps \textsf{Filter-Price} with
\textsf{Rank} across the dependency edge between them, so the agent
ranks results before the price filter that feeds the ranking exists.

\paragraph{Dependency matrices.}
Agent-side graphs are re-extracted from the perturbed sequences by
artifact tracing. For P2 the traced graph is unchanged, so its signed
dependency matrix (Eq.~\ref{eq:depdist}) equals the reference matrix
exactly:
\[
D^{R} = D^{P2} = \scriptsize
\begin{pmatrix}
0 & .25 & .25 & .50 & .75 & 1 \\
-.25 & 0 & 0 & .25 & .50 & .75 \\
-.25 & 0 & 0 & .25 & .50 & .75 \\
-.50 & -.25 & -.25 & 0 & .25 & .50 \\
-.75 & -.50 & -.50 & -.25 & 0 & .25 \\
-1 & -.75 & -.75 & -.50 & -.25 & 0
\end{pmatrix}.
\]
For P4, tracing cannot satisfy the inverted dependency, deletes the
edge, and the chain fragments (rows and columns ordered as executed:
\textsf{Search}, \textsf{Rank}, \textsf{Filter-Loc},
\textsf{Filter-Price}, \textsf{Select}, \textsf{Book}):
\[
D^{P4} = \scriptsize
\begin{pmatrix}
0 & 0 & .50 & .50 & 0 & 0 \\
0 & 0 & 0 & 0 & .50 & 1 \\
-.50 & 0 & 0 & 0 & 0 & 0 \\
-.50 & 0 & 0 & 0 & 0 & 0 \\
0 & -.50 & 0 & 0 & 0 & .50 \\
0 & -1 & 0 & 0 & -.50 & 0
\end{pmatrix}.
\]

\paragraph{Optimal couplings.}
Both couplings (entries $\times 100$; rows are agent steps in executed
order, columns are reference steps) still assign essentially all of
each step's mass to its semantic counterpart --- the row-argmax of
$\T^\star$ is a readable one-to-one step mapping in both cases:
\[
\T^{P2} \approx \scriptsize
\begin{pmatrix}
15.9 & .1 & .1 & 0 & 0 & .1 \\
0 & .9 & 14.7 & .2 & .1 & .7 \\
0 & 14.8 & .9 & .1 & .1 & .6 \\
0 & 0 & .1 & 15.1 & .1 & 1.6 \\
0 & 0 & 0 & 0 & 14.6 & 2.8 \\
0 & 0 & 0 & 0 & 0 & 19.9
\end{pmatrix},
\]
\[
\T^{P4} \approx \scriptsize
\begin{pmatrix}
15.1 & .1 & .1 & .1 & .1 & .5 \\
.1 & .2 & .3 & 15.2 & .1 & .8 \\
0 & .8 & 13.7 & .1 & .1 & 1.9 \\
0 & 14.0 & .9 & .1 & .1 & 1.6 \\
0 & .1 & .1 & .1 & 14.5 & 2.1 \\
0 & 0 & 0 & 0 & 0 & 18.8
\end{pmatrix}.
\]

\paragraph{Sub-losses.}
The structural term accounts for nearly all of the increase in total
loss; the semantic and marginal-KL changes largely offset one another:

\begin{center}
\scriptsize
\setlength{\tabcolsep}{3pt}
\begin{tabular}{lccccc}
\toprule
Variant & $(1{-}\theta)\langle \Cost,\T\rangle$ & $\theta\,E_{\mathrm{GW}}$ & KL & $\mathcal{L}$ & score \\
\midrule
original & 0.0195 & 0.0028 & 0.0077 & 0.0300 & 0.9705 \\
P2 (valid reorder) & 0.0195 & 0.0028 & 0.0077 & 0.0300 & 0.9705 \\
P4 (inversion) & 0.0239 & 0.0231 & 0.0045 & 0.0515 & 0.9498 \\
\bottomrule
\end{tabular}
\end{center}

Three mechanisms are visible. First, the original and P2 rows are
identical to all reported digits: Proposition~\ref{prop:perm}'s
invariance holds exactly, because $D^{P2}=D^R$ and the node costs are
permutations of the same entries. Second, for P4 the semantic term
barely moves (the coupling still finds every step's counterpart), while
the structural term rises by a factor of eight
($0.0028 \to 0.0231$): the penalty is attributable to the dependency
violation, not to any text change --- an inversion changes no words.
This is precisely the failure mode that text-only metrics cannot see.
Third, scores do not reach $1.0$ even for the original, because
entropic smoothing spreads a small amount of mass off the diagonal;
the original and P2 have the same reported loss here, but entropy-induced
smoothing effects need not cancel for arbitrary candidates. The example is reproducible with the released
harness.

\paragraph{Attribution beyond this example.}
The read-out that this example illustrates can be quantified. For each
of the $120$ damaged variants in the corpus we compute the change of
every sub-loss relative to the task's original trajectory, standardize
each channel on the benign variants, and predict the defect family from
the channel with the largest standardized change. Semantic defects
(incompatible tool swaps) and structural defects (causal inversions)
are never confused: all $79$ such items are attributed correctly.
Coverage defects (deletions and insertions) are attributed correctly in
only $13$ of $41$ cases, because a deletion's signature --- one node
fewer, a changed dependency matrix, unmatched reference mass ---
overlaps that of a benign merge (the granularity limitation discussed
in the main paper), and a deletion also genuinely breaks dependencies,
so the boundary between the coverage and structural categories is
partly definitional. Even when the three-way defect label is ambiguous, the coupling still
shows \emph{where} the problem is: for $64\%$ of damaged trajectories,
at least one individual step is left with more unmatched mass than
typically occurs on benign variants (above the $95$th percentile
of the benign calibration), identifying steps for inspection even without
a reliable defect-category label. This exceedance statistic alone does
not establish that the flagged step is the true defect location.

\section{Criteria, Violating Metric Families, and Tests}
\label{app:criteria}

Each criterion of the Introduction corresponds to a documented failure
of an existing metric family; Table~\ref{tab:criteria} names the family
and the experiment that tests the criterion.

\begin{table}[h]
\centering
\caption{Each criterion, a metric family that violates it, and the
experiment that tests it. Numbers refer to Tables~\ref{tab:validity}
and~\ref{tab:ablations}.}
\label{tab:criteria}
\small
\setlength{\tabcolsep}{3pt}
\begin{tabular}{p{2.5cm}p{2.6cm}p{2.6cm}}
\toprule
Criterion & Violated by & Tested by \\
\midrule
(i) semantic awareness & lexical overlap (exact 36.4, BLEU 49.2 on P1) & P1 paraphrase \\
(ii) permutation tolerance + causal strictness & ROUGE-L (P2 45.8); order-blind embeddings (P4 25.4) & P2 vs.\ P4 \\
(iii) granularity tolerance & exact match, one-to-one matchers (P3 $\le$ 45.8) & P3 (open for \otap{} too: 54.7) \\
(iv) cardinality robustness & Hungarian matching (needs dummy padding) & P5 delete / insert \\
(v) execution grounding & all text-only metrics & P5 tool swap; no-tool ablation (68.9); state term untested \\
(vi) multi-reference support & single-gold assumptions & $K{=}1$ ablation (65.1 vs.\ 82.0) \\
\bottomrule
\end{tabular}
\end{table}

\end{document}